\documentclass{article}
\usepackage{ijcai15}

\usepackage{times}

\usepackage{latexsym} 
\usepackage{amsmath}
\usepackage{amsthm}
\usepackage{amssymb}
\usepackage{graphicx}
\usepackage{caption}
\usepackage{subcaption}
\usepackage{subfig}
\usepackage{float}
\usepackage{epsfig}

\newtheorem{fact}{Fact}
\newtheorem{lemma}[fact]{Lemma}
\newtheorem{theorem}[fact]{Theorem}

\title{Further Connections Between Contract-Scheduling and Ray-Searching Problems\thanks{Research supported by project 
ANR-11-BS02-0015 ``New Techniques in Online Computation--NeTOC''.} }

\author{Spyros Angelopoulos \\
Sorbonne Universit\'{e}s, UPMC Univ Paris 06, UMR 7606, LIP6, F-75005, Paris, France and \\
CNRS, UMR 7606, LIP6, F-75005, Paris, France \\
{spyros.angelopoulos@lip6.fr}
}

\begin{document}

\maketitle

\begin{abstract}
 This paper addresses two classes of different, yet interrelated optimization problems. 
The first class of problems involves a robot that must locate a hidden target in 
an environment that consists of a set of concurrent rays. The second class 
pertains to the design of interruptible algorithms by means of a schedule of contract 
algorithms. We study several variants of these families of problems, such as 
searching and scheduling with probabilistic considerations, redundancy and fault-tolerance issues,
randomized strategies, and trade-offs between performance and preemptions. For many of these
problems we present the first known results that apply to multi-ray and multi-problem domains.
Our objective is to demonstrate that several well-motivated settings can be 
addressed using the same underlying approach. 
\end{abstract}
\section{Introduction}
\label{sec:introduction}

In this paper we expand the study of connections between two seemingly different,  
yet interrelated classes of scheduling problems. The first class of problems involves a mobile searcher
that must explore an unknown environment so as to locate a hidden target. 
Objectives of this nature are often encountered in the domain of robotic search and exploration. 
The second class of problem pertains to the design of a computational multi-problem solver, which 
may be interrupted at any point in time, and may be queried for its currently best solution 
to any of the given problems. This setting provides a very practical modeling of situations that
often arise in the realm of AI applications, such as the design of any-time and real-time intelligent 
systems~\cite{Zaimag96}.

Searching for a hidden object in an unbounded domain is a fundamental computational problem, with a rich 
history that dates back to early work~\cite{bellman}~\cite{beck:ls} in the context of searching on the infinite line (informally known 
as the {\em cow-path problem}). In our work we focus on a generalization of linear search, known as
the {\em star search} or {\em ray search} problem. Here, 
we are given a set of $m$ semi-infinite, concurrent rays which intersect at a common origin $O$, as well as
a mobile searcher which is initially placed at the origin. 
There is also a target that is hidden at some distance $d$ from $O$, at a ray unknown to the searcher. The objective is 
to design a search strategy that minimizes the {\em competitive ratio}, namely the worst-case ratio 
of the distance traversed by the robot (up to target detection) over the distance $d$.

Problems related to ray searching have attracted significant interest from the
AI/OR communities. 
Optimal competitive ratios were obtained 
in~\cite{gal:minimax} and~\cite{yates:plane}.
The setting in which certain probabilistic information concerning the target placement is known was studied 
in~\cite{jaillet:online},~\cite{informed.cows}. The effect of randomization on the expected 
performance was addressed in~\cite{schuierer:randomized},~\cite{ray:2randomized}. In the case where
an upper bound  on the distance from the target is known~\cite{ultimate} provides a near-optimal asymptotic analysis, 
whereas in the case where the searcher incurs a fixed turn cost~\cite{demaine:turn} provides an optimal
search strategy. Other work includes the setting of multiple parallel searchers~\cite{alex:robots}, the
related problem of designing {\em hybrid algorithms}~\cite{hybrid}, 
and more recently, the study of new performance  measures~\cite{hyperbolic},~\cite{oil}.
We refer the interested reader to Chapters 8 and 9 in~\cite{searchgames} for further results. 

The second class of problems is related to bounded-resource reasoning in the context of 
{\em anytime} algorithms~\cite{RZ.1991.composing}. 
Such algorithms provide a useful trade-off between computation 
time and the quality of the output, when there is uncertainty with respect to the allowed
execution time. More specifically, our goal is to be 
able to simulate an {\em interruptible algorithm} by means of repeated executions of a {\em contract algorithm}.
These are both classes of anytime algorithms which, however, differ significantly in terms of their 
handling of interruptions. On the one hand, an interruptible algorithm will always produce some meaningful result
(in accordance to its performance profile) whenever an interruption occurs during its execution.  
On the other hand, a contract algorithm must be provided, as part of the input, with its pre-specified
computation time (i.e., contract time). If completed by the contract time, the algorithm will 
always output the solution consistent with its performance profile, otherwise 
it may fail to produce any useful result. 

As observed in~\cite{BPZF.2002.scheduling}, contract  algorithms tend to be simpler to implement and maintain, 
however they lack in flexibility compared to interruptible algorithms. This observation raises the challenge 
of simulating an interruptible algorithm using repeated executions of contract algorithms. The precise
framework is as follows: given $n$ instances of optimization problems, and a contract algorithm for each problem, 
provide a strategy for scheduling repeated executions of a contract algorithm, in either 
a single, or multiple processors. Upon an interruption, say at time $t$, the solution to any of the $n$ problems may
be requested. The system returns the solution that corresponds to the longest completed execution of a contract algorithm
for the problem in question. The standard performance measure of this scheduling
strategy is the {\em acceleration ratio}~\cite{RZ.1991.composing}, which informally 
can be described as a resource-augmentation measure:
namely, it implies that an increase of the processor speed by a 
factor equal to the acceleration ratio of the schedule yields a system which is as efficient as 
one in which the interruption time is known in advance.

Previous research has established the optimality of scheduling strategies based on iterative deepening methods
in the settings of single problem/single processor~\cite{RZ.1991.composing}~\cite{ZCC.1999.realtime},
single problem/multiple processors~\cite{ZCC.1999.realtime} and 
multiple problems/single processor~\cite{BPZF.2002.scheduling}. The most general setting of multiple problems and 
processors was investigated in~\cite{steins}, which was also the first to demonstrate connections between ray
searching and contract scheduling problems. More specifically~\cite{steins} shows that a reduction between specific 
classes of search and scheduling strategies known as {\em cyclic} strategies (see Section~\ref{sec:preliminaries}).
Optimal schedules, without restrictions, were established in~\cite{aaai06:contracts}.
Issues related to soft deadlines were addressed in~\cite{soft-contracts}, and measures alternative to the acceleration 
ratio have been introduced in~\cite{ALO:multiproblem}. 

\smallskip
\noindent
{\bf Contribution of this paper} \ 
In this work we expand the study of connections between the search and scheduling problems that was initiated 
in~\cite{steins}. Namely, we address several settings that provide well-motivated 
extensions and generalizations of these two classes of problems. More precisely, we study the following 
problems:

\noindent
{\em  Uncertain target detection / Monte Carlo contract algorithms:} \ We investigate the setting in which the 
searcher detects the target with probability $p$ during each visit, and the setting in which each contract
algorithm is a randomized Monte Carlo algorithm with probability of success equal to $p$.

\noindent
{\em Redundancy and fault tolerance:} \ 
We seek search strategies under the constraint that at least $r$ visits over the target are required in 
order to locate it. On a similar vain, we seek scheduling strategies under the assumption that at least $r$ executions
of a contract algorithm are required so as to benefit from its output. This is related to search and 
scheduling with uncertainty, when the probability of success is unknown.

\noindent
{\em  Randomized scheduling strategies:} \ We show how access to random bits can improve the expected performance 
of a scheduling strategy.

\noindent
{\em  Trade-offs between performance and the number of searches and contracts:} \ 
We quantify the trade-offs between the performance ratios and the number of turns by the searcher or
the number of algorithm executions in the schedule. 
 
For all problems, with the exception of randomized strategies, 
we give the first results (to our knowledge) that 
apply to both the multi-ray searching and multi-problem scheduling domains. 
Concerning randomization, we show how to apply and extend, in a non-trivial manner, 
ideas that stem from known randomized ray-searching algorithms. 
In addition, we address an open question in~\cite{steins}, who asked ``whether 
the contract scheduling and robot search problems have similarities beyond those that result from using
cyclic strategies''. In particular, in Section~\ref{sec:fault} we present non-cyclic strategies that improve
upon the best cyclic ones. 
 


\section{Preliminaries}
\label{sec:preliminaries}

\noindent
{\em Ray searching.} We assume a single robot and $m$ rays, numbered $0 \ldots m-1$. 
For a target placement $T$ at distance $d$ from the origin, we define the {\em competitive ratio} 
of a strategy as
\begin{equation}
\alpha=\sup_T \frac{\mbox{cost for locating T}}{d}
\label{eq:compititive.ratio}
\end{equation} 
A strategy is  {\em round-robin} or {\em cyclic} if it described by an infinite sequence 
$\{x_i\}_{i=0}^\infty$ as follows: in the $i$-th iteration, 
the searcher explores ray $(i \bmod m)$ by starting at the origin $O$, reaching
the point at distance $x_i$ from $O$, and then returning to $O$. A cyclic strategy is called {\em monotone}, if
the sequence  $\{x_i\}_{i=0}^\infty$ is non-decreasing. A special class of monotone strategies is the class
of {\em exponential} strategies, namely strategies in which $x_i=b^i$, for some given $b>1$, which we call the {\em base}
of the strategy. Exponential strategies are often optimal among monotone strategies
(see~\cite{searchgames}), and in many cases they are also globally optimal. Indeed, for $m$-ray searching,
the exponential strategy with base $b=\frac{m}{m-1}$ attains the optimal competitive ratio~\cite{gal:general}
\begin{equation}
\alpha^*(m)=1+2\frac{b^m-1}{b-1}, \ b=\frac{m}{m-1}.    
\label{eq:prelim.optimal.cr}
\end{equation} 
Note that $\alpha^*(m)=O(m)$, and $\alpha^*(m) \rightarrow 1+2\mathrm e m$ as $m\rightarrow \infty$.

\noindent
{\em Contract scheduling:} We assume a single processor and $n$ problems, numbered $0 \ldots n-1$. 
For interruption time $t$, let $\ell_{i,t}$ denote the length (duration) of the longest execution of a contract 
algorithm for problem $i$ that has completed by time $t$. Then the acceleration ratio 
of the schedule~\cite{RZ.1991.composing} is defined as 
\begin{equation}
\beta= \sup _{t,  i \in [0,\ldots, n-1]} \frac{t}{\ell_{i,t}}. 
\label{eq:acceleration.ratio}
\end{equation}
Similar to ray searching, a round-robin or cyclic strategy is described by an infinite sequence 
$\{x_i\}_{i=0}^\infty$ such that in iteration $i$, the strategy schedules an execution of 
a contract for problem $(i \bmod n)$, and of length equal to $x_i$. The definitions of monotone and exponential strategies
are as in the context of ray searching, and we note that, once again, exponential strategies often lead
to optimal or near-optimal solutions (see, e.g.,~\cite{ZCC.1999.realtime}~\cite{aaai06:contracts},~\cite{soft-contracts}). 
In particular, for $n$ problems, the exponential strategy with base $b=\frac{n+1}{n}$ attains the optimal
acceleration ratio~\cite{ZCC.1999.realtime}
\begin{equation}
\beta^*(n)=\frac{b^{n+1}}{b-1}, \ b=\frac{n+1}{n}. 
\label{eq:prelim.optimal.ar}
\end{equation}
Note that $\beta^*(n)=O(n)$, and that $\beta^*(n) \rightarrow \mathrm e(n+1)$, for $n\rightarrow \infty$.

Occasionally, we will make a further distinction between worst-case and asymptotic performance. Namely, the 
asymptotic competitive ratio is defined as $\lim_{T:d \rightarrow \infty} \frac{\mbox{cost for locating T}}{d}$,
whereas the asymptotic acceleration ratio is defined as 
$\lim _{t \rightarrow \infty} \sup_{i \in [0,\ldots, n-1]} \frac{t}{l_{i,t}}$ (assuming that the measures converge to a limit).

\section{Search with probabilistic detection and scheduling of randomized contracts}
\label{sec:probabilistic}

In this section we study the effect of uncertainty in search and scheduling.
In particular, we consider the setting in which the detection of a target is stochastic,
in that the target is revealed with probability $p$ every time the searcher passes over it. 
Similarly, we address the problem of scheduling randomized contract algorithms; namely, each 
execution of the (Monte Carlo) randomized algorithm succeeds with probability $p$. This variant
has been studied in~\cite{searchgames} only in the context of linear search (i.e., when $m=2$),
and the exact competitiveness of the problem is not known even in this much simpler case. 
 No results are known for general $m$. 

In this setting, the search cost is defined as the expected time of the first successful target detection.
Moreover, for every problem $i$ and interruption $t$, we define $\mathbb E[\ell_{i,t}]$ as
the expected longest contract completed for problem $i$ by time $t$. The competitive and the 
acceleration ratios are then defined naturally as extensions of~(\ref{eq:compititive.ratio})
and~(\ref{eq:acceleration.ratio}). 

We begin with a lower bound on our measures.
\begin{lemma}
Every search strategy with probabilistic detection has competitive ratio at least $\frac{m}{2p}$,
and every scheduling strategy of randomized contract algorithms has acceleration ratio at least
$\frac{n}{p}$.
\label{lemma:probabilistic.lower}
\end{lemma}
\begin{proof}
Consider first the search variant. Let $S$ denote the set of all points at distance at most $d$
from the origin. Given a search strategy and a point $x \in S$, let 
$t_x^k$ denote the time in which the searcher reaches $x$ for the $k$-th time. We will first show 
that for every $k \geq 1$, there exists $x \in S$ such that the search cost at the time 
of the $k$-th visit of $x$ is at least $kmd/2$.
To this end, we will need the assumption that the searcher cannot perform infinitely small oscillations around 
a point. More precisely, we will assume that, for arbitrarily small but fixed $\epsilon>0$, 
if the searcher visits a point that belongs in an interval of length
$\epsilon$ on some ray, then it must leave the interval before re-visiting this point in the future. This assumption
is required for technical reasons, but also applies naturally to robotic search. Consider the partition of
all points in $S$ in intervals of length $\epsilon$; for each such interval $I$ denote by $c_I$ 
the point in the middle of interval. For given $c_I$ the searcher needs to enter $I$, visit $c_I$ and eventually leave the interval $I$ 
$k$ times, which incurs a cost of at least $\epsilon\frac{k}{2}$. Therefore, the overall cost for visiting each center $k$ times is
at least $kmd/2$, which further implies that there exists a point in $S$ with the desired property (namely, the center whose 
$k$-th visit occurs last).

Given the above bound, we obtain that targets in $S$ are detected at expected cost at least
$\sum_{k=1}^\infty p(1-p)^{k-1} t_x^k \geq \sum_{k=1}^\infty p(1-p)^{k-1} kmd/2 =md/(2p)$.
The result follows directly from~\eqref{eq:compititive.ratio}.

Consider now the scheduling variant. For a given interruption time $t$ and a given problem 
instance $i$, let $l_1^i, l_2^i, \ldots l_{n_i}^i$
denote the lengths of the contracts for problem $i$ that have completed by time $t$, in non-increasing
order. Let the random variable $\ell_{i,t}$ denote the expected length of the 
longest contract completed for problem $i$ by time $t$. 
Then $\mathbb E[\ell_{i,t}]=\sum_{j=1}^{n_i} p(1-p)^{j-1}  l_{j}^{i} \leq p \sum_{j=1}^{n_i} l_j^i$.
Since $\sum_{i=0}^{n-1} \sum_{j=1}^{n_i} l_j^{i} =t$, there exists a problem 
$i$ for which $\mathbb E[\ell_{i,t}] \leq p \frac{t}{n}$. The claim follows from the definition
of acceleration ratio~(\ref{eq:acceleration.ratio}).
\end{proof}

\begin{theorem}
There exists an exponential strategy for searching with probabilistic detection that has
competitive ratio at most $1+8\frac{m}{p^2}$. 
\label{thm:probabilistic.search.upper}
\end{theorem}

\begin{proof}
Let $\{x_i\}_{i=0}^\infty$ denote the searcher's exponential 
strategy, where $x_i=b^i$, for some $b$ that will be chosen later in the proof.    
Let $d$ denote the distance of the target from the origin, then there exists index $l$
such that $x_l <d \leq x_{l+m}$. We denote by $P_k$ the probability that the target is found 
during the $k$-th visit of the searcher, when all previous $k-1$ attempts were unsuccessful,
hence $P_k =(1-p)^{k-1}p$. We also define $q_j \stackrel\cdot{=}\sum_{k=j}^\infty P_k=(1-p)^{j-1}$.

In order to simplify the analysis, we will make the assumption that the searcher can locate 
the target only while it is moving away from the
origin (and never while moving towards the origin); it turns out that this assumption
weakens the result only by a constant multiplicative factor. 

We first derive an expression for the expected total cost $C$ incurred by the strategy.
Note that first time the searcher passes the target, it has traveled a total distance of at most 
$2\sum_{j=0}^{l+m-1}x_j+d$; more generally, the total distance traversed by the searcher at its
$k$-th visit over the target is at most $2\sum_{j=0}^{l+km-1}x_j+d$. We obtain that the expected cost
is bounded by 
\[
C=\sum_{k=1}^\infty P_k (2 \sum_{j=0}^{l+mk-1} x_j+d), 
\]
from which we further derive (using the connection between $P_k$ and $q_j$) that the competitive ratio of the strategy is 
\begin{eqnarray}
\alpha &\leq& \frac{C}{d} \leq 1+ \frac{2}{x_l}\sum_{k=1}^{\infty} P_k \sum_{j=0}^{l+mk-1} x_j  \nonumber \\
&=& 1+\frac{2}{x_l} \sum_{j=0}^l x_{j+m-1} + \frac{2}{x_l} \sum_{j=2}^\infty q_j \nonumber
\sum_{i=1}^{m-1} x_{l+(j-1)m+i}.
\end{eqnarray}
By rearranging the terms in the summations we observe that 
\begin{eqnarray}
&&\sum_{j=2}^\infty q_j \sum_{i=1}^{m-1} x_{l+(j-1)m+i} =
\sum_{i=1}^{m-1} \sum_{j=2}^\infty q_j x_{l+(j-1)m+i}  \nonumber  \\ 
&=&
\sum_{i=1}^{m-1} x_{l+i} \sum_{j=2}^\infty ((b^m(1-p))^{j-1}.
\end{eqnarray}

By defining $\lambda \stackrel\cdot{=} b^m(1-p)$, and by 
combining the above inequalities
we obtain that the competitive ratio is at most
$\alpha \leq 1+ 2\frac{b^m}{b-1} \sum_{j=0}^\infty \lambda^j$.
Note that unless $\lambda<1$ the competitive ratio is not bounded. Assuming that we can choose $b>1$ such that 
$\lambda<1$, the competitive ratio is 
\begin{equation}
\alpha \leq 1+ 2\frac{b^m}{b-1} \cdot \frac{1}{1-\lambda}.
\label{eq:prob.search.upper.4}
\end{equation}
We will show how to choose the appropriate $b>1$ so as to guarantee the desired competitive ratio. 
To this end, we will first need the following technical lemma.
\begin{lemma}
The function $f:\mathbb R^+ \rightarrow \mathbb R$ with $f(x)= \mathrm e^x(1-p)+\mathrm e^x \frac{p^2}{4x}-1$
has a root $r$ such that $0<r\leq \frac{p}{2}$.
\label{lemma:geometric.technical}
\end{lemma}
\begin{proof}
The function $f$ is continuous in the interval $(0,+\infty)$, and for $x \rightarrow 0^+$, $f(x)>0$.
Suffices to show that there exists $y \leq \frac{p}{2}$ such that $f(y) \leq 0$; then the existence of the desired
root follows from Bolzano's theorem. For all $x<1$ we have 
\begin{eqnarray}
f(x) &\leq& \frac{1}{1-x}\left(1-p+\frac{p^2}{4x}\right)-1, \ 
 \mbox{ since} \ \mathrm e^x \leq \frac{1}{1-x} \nonumber \\ 
&\leq& \frac{\left(x-\frac{p}{2} \right)^2}{x(1-x)}.
\end{eqnarray} 
Choosing $y_0=\frac{p}{2}$, we obtain that $f(y) \leq 0$, and the lemma follows.
\end{proof}

Let $r$ denote the root of the function $f$, defined in the statement of Lemma~\ref{lemma:geometric.technical}.
We will show that choosing base $b=\frac{m}{m-r}$ yields the desired competitive ratio. 
It is straightforward to verify that $b>1$ and that $\lambda=b^m(1-p)<1$.
Hence, the competitive ratio converges to the value given by the RHS of~\eqref{eq:prob.search.upper.4}.
From the choice of $b$, we have that $b-1=\frac{r}{m-r}$ and $b^m \leq e^{r}$.
We then obtain 
$
 \frac{b^m}{b-1} \cdot \frac{1}{1-\lambda} \leq \frac{e^{r}(m-r)}{r\left((1-(1-p) \left(\frac{m}{m-r}\right)^m\right)}
\leq \frac{me^{r}}{r\left((1-(1-p) \mathrm e ^{r}\right)}.
$
Recall that from Lemma~\ref{lemma:geometric.technical}, $r$ is such that
$1-(1-p)\mathrm e^{r}=\mathrm e^{r} \frac{p^2}{4}$. We thus obtain that 
$\frac{b^m}{b-1}\frac{1}{1-\lambda} \leq \frac{m}{4p^2}$, and from~(\ref{eq:prob.search.upper.4}) 
it follows that the competitive ratio of the strategy is at most 
$1+8 m/p^2$.
\end{proof}

\begin{theorem}
\label{thm:geometric.contracts}
There exists an exponential strategy for scheduling randomized contract algorithms that has acceleration
ratio at most $\mathrm e \frac{n}{p}+ \frac{\mathrm e}{p}$.
\end{theorem}

\begin{proof}
Let $b$ denote the base of the exponential strategy. It is easy to see that the acceleration ratio is maximized for 
interruptions $t$ that are arbitrarily close to, but do not exceed the finish time of a contract. Let $t$ denote 
such an interruption time, in particular right before termination of contract $i+n$, for some $i>0$; 
in other words, $t=\frac{b^{i+n+1}-1}{b-1}$.
Then every problem has completed a contract of expected length at least 
$p b^i$ by time $t$. Therefore, the acceleration ratio of the schedule is at most 
$
\beta \leq \sup_{i>0} \frac{b^{n+i+1}}{pb^i(b-1)},
$
and choosing $b=\frac{n+1}{n}$ we obtain that $\beta \leq 
\mathrm e \frac{n}{p}+ \frac{\mathrm e}{p}$.
A more careful analysis of the same strategy yields a better asymptotic acceleration ratio. More specifically,
it is easy to see that for interruption $t$ defined as above and for every problem $j$, 
the strategy has completed a contract for problem $j$ of expected length at least 
$\sum_{l=0}^k p(1-p)^lb^{i-nl}$, where
$k$ is such that $i-kn=i \bmod n$. It follows that the acceleration ratio
is at most $\frac{b^{n+1}}{p(b-1)} \cdot \frac{1}{\sum_{l=0}^k \left( \frac{1-p}{b^n} \right)^l}$.
Choosing again $b=\frac{n+1}{n}$, and after some simple calculations, we have that the asymptotic acceleration ratio
(obtained for $k \rightarrow \infty$), is at most $(\mathrm e-1+p) \frac{n}{p}+O(\frac{1}{p})$.
\end{proof} 

\section{Fault tolerance/redundancy in search and scheduling}
\label{sec:fault}

In Section~\ref{sec:probabilistic} we studied the searching and scheduling problems in a stochastic setting. 
But what if the success probability is not known in advance? In the absence of such information, one could 
opt for imposing a lower bound $r$ on the number of times the searcher has to visit the target and, likewise, 
a lower bound $r$ on the number of times a contract algorithm must be executed before its response can be trusted.
Alternatively, this setting addresses the issues of fault tolerance and redundancy in the search and 
scheduling domains. The search variant has been studied in~\cite{searchgames} only in the context of linear 
search ($m=2$); as in the case of probabilistic detection, even when $m=2$ the exact optimal competitive strategies
are not known. 

The following lemma follows using an approach similar to the proof 
of Lemma~\ref{lemma:probabilistic.lower}.
\begin{lemma}
Every search strategy on $m$ rays with redundancy guarantee $r \in \mathbb N^+$ has competitive ratio at least $\frac{rm}{2}$.
\label{thm:fault.lower}
\end{lemma}

We first evaluate the best exponential strategy.
\begin{theorem}
The best exponential strategy has competitive ratio at most 
$
2(\left \lceil\frac{r}{2}\right \rceil m-1)\ \left(\frac{\lceil\frac{r}{2}\rceil m}{\lceil\frac{r}{2}\rceil m-1}
\right)^{\lceil\frac{r}{2}\rceil m} + 1 \leq 2\mathrm e ( \lceil \frac{r}{2} \rceil(m-1)) + 1.
$
\label{thm:fault.geometric}
\end{theorem}
\begin{proof}
Let $\{x_i\}_{i=0}^\infty $ denote the exponential strategy, with $x_i=b^i$ for some $b$ to be fixed later. 
Suppose that the target is at distance $d$ from the origin, and let $l \in \mathbb N$ be such that 
$x_l <d \leq x_{l+m}$. We need to consider cases concerning the parity of $r$.
If $r$ is odd, i.e., $r=2k+1$ for $k \in \mathbb N$, then the cost of the strategy is upper bounded by 
$
2\sum_{i=0}^l x_i +2\sum_{i=1}^{(k+1)m-1} x_{l+i}+d = 2\sum_{i=0}^{(k+1)m-1} x_{i}+d, 
$
whereas if $r$ is even, ie. $r=2k$, the cost is bounded by
$
2\sum_{i=1}^l x_i +2\sum_{i=1}^{km-1} x_{l+i}+(x_{km}-d)= 2\sum_{i=0}^{km}x_i-d. 
$
It follows that the competitive ratio of the exponential strategy is at most
$1+2\frac{b^\frac{(r+1)m}{2}-1}{b-1}$, if $r$ is odd, and at most $2\frac{b^\frac{rm}{2}-1}{b-1}-1$ if $r$ is even.
We observe that in both cases, the competitive ratio is essentially identical to the competitive ratio 
of an exponential strategy  with base $b$, when searching for a single target in $\lceil\frac{r}{2}\rceil m$ rays without 
fault-tolerance considerations (with the exception of the negligible additive unit terms).
This motivates the choice of $b=\frac{\lceil\frac{r}{2}\rceil m}{\lceil\frac{r}{2}\rceil m-1}$ as the optimal base of the exponential strategy, 
which yields  a competitive ratio equal to $2(\lceil\frac{r}{2}\rceil m-1)\ \left(\frac{\lceil\frac{r}{2}\rceil m}
{\lceil\frac{r}{2}\rceil m-1}\right)^{\lceil\frac{r}{2}\rceil m} + 1 \leq 
2\mathrm e (\lceil\frac{r}{2}\rceil m-1) + 1$, if $r$ is odd, and
$2(\lceil\frac{r}{2}\rceil m-1)\ \left(\frac{\lceil\frac{r}{2}\rceil m}
{\lceil\frac{r}{2}\rceil m-1}\right)^{\lceil\frac{r}{2}\rceil m} - 1 \leq 
2\mathrm e (\lceil\frac{r}{2}\rceil m-1) - 1$, if $r$ is even. 
\end{proof}

Interestingly, we can show that there exist non-monotone strategies, which, for $r>2$, 
improve upon the (best) exponential strategy of Theorem~\ref{thm:fault.geometric}. For simplicity, let us assume
that $r$ is even, although the same approach applies when $r$ is odd, and leads to identical results up to an 
additive constant. In particular, we will consider the following strategy: In iteration $i$, the searcher visits 
ray $i \bmod m$ first up to the point at distance  $x_{i-m}$, then performs $r$ traversals of the interval 
$[x_{i-m},x_i]$ (thus visiting $r$ times each point of the said interval), then completes the iteration by 
returning to the origin (here we define $x_j=0$ for all $j<0$). We call this strategy {\sc NM-search} (non-monotone search).

\begin{theorem}
Strategy {\sc NM-search} has competitive ratio at most 
$r(m-1) \left( \frac{m}{m-1}\right)^m+2-r.$ 
\label{fault:nm}
\end{theorem}
\begin{proof}
Suppose that the target lies at a distance $d$ from the origin, and let $l \in N$ denote an index such that 
$x_l <d \leq x_{l+m}$. Then the cost of locating the target is at most
\[
\sum_{j=0}^{l+m} (r(x_j-x_{j-m})+2x_{j-m})=
r \cdot \sum_{j=0}^{m+l} x_{j}+(2-r)\sum_{j=0}^{l} x_j.
\] 
Setting $x_i=b^i$ (which we will fix shortly), and given that $d>x_l$, we obtain that the competitive ratio is at most
\begin{equation}
\alpha \leq r b \frac{b^{m+1}}{b-1} +(2-r)\frac{b^{l}-1}{b^l(b-1)} 
\leq \frac{rb^{m+1}}{b-1}+(2-r),
\label{eq:fault.nm.1}
\end{equation}
where the last inequality follows from the fact that $b^l>1$. We now observe that~(\ref{eq:fault.nm.1}) is 
minimized for $b=\frac{m}{m-1}$. Substituting in~(\ref{eq:fault.nm.1}) yields
\[
\alpha \leq
r(m-1) \left( \frac{m}{m-1}\right)^m+2-r.
\]
\end{proof}

It is very easy to show, by comparing the results of Theorems~\ref{thm:fault.geometric} and~\ref{fault:nm}, that 
the non-monotone strategy is superior to the best exponential strategy for $r>2$.


Consider now contract scheduling with redundancy parameter $r$, in the sense that the interruptible system 
may output only the solutions of contracts that have been executed at least $r$ times by time $t$. In this setting, the 
best schedule is derived from a pseudo-exponential strategy, which is defined in phases as follows: in phase $i \geq 0$,
$r$ contracts for problem $i\mod n$, and of length $b^i$ are executed, for given base $b>i$. 
It turns out that this strategy attains the optimal acceleration ratio. The proof of the following theorem uses
techniques from~\cite{aaai06:contracts}.
\begin{theorem}
The pseudo-exponential scheduling strategy with base $b=\frac{n+1}{n}$ has acceleration ratio at most
$rn \left(\frac{n+1}{n} \right)^{n+1}$. Furthermore, this acceleration ratio is optimal. 
\label{thm:fault.contract}
\end{theorem}
\begin{proof}
The pseudo-exponential strategy with base $b$ can be analyzed using the standard approach (e.g. as in~\cite{ZCC.1999.realtime}),
and its acceleration ratio is equal to $r \frac{b^{n+1}}{b-1}$, which is minimized for $b=\frac{n+1}{n}$. On the other hand, the 
lower bound follows based on ideas very similar to~\cite{aaai06:contracts}, which gives a tight lower bound on the acceleration
ratio of every schedule. In particular, the crucial observation is that there exists an optimal schedule with the property that 
whenever a new contract is about to be scheduled, the problem with the smallest completed contract length (where completion 
now is defined to multiplicity $r$) will be chosen. The remaining technical details follow precisely along the lines of the 
proof of Theorem 1 in~\cite{aaai06:contracts}.  
\end{proof}

A different setting stipulates that the schedule returns, upon interruption $t$ and for queried problem $p$, the $r$-th smallest
contract for problem $p$ that has completed its execution by time $t$. In this setting, we can still apply the pseudo-exponential
strategy (which is clearly non-monotone). We can show, as in ray searching, that this strategy is better than the 
best exponential strategy, albeit slightly so. 
\begin{theorem}
The best exponential strategy for the $n$-problem contract scheduling problem with redundancy parameter $r$
has acceleration ratio at most 
$
\left(rn+1\right)\left(1+\frac{1}{rn}\right)^{rn} \leq \mathrm ern+e.
$ 
Furthermore, there exists a non-monotone strategy which improves upon the best exponential strategy for all $n,r$.  
\end{theorem}
\begin{proof}
Let $b$ denote the base of the exponential strategy. Consider a worst-case interruption at time $t$, right before the 
end of the $n+i$-th contract, i.e. at time $t=\frac{b^{n+i+1}-1}{b-1}$. Then, for every problem $p$, the scheduler 
has completed $r$ contracts for $p$ at lengths at least $b^{i-(r-1)n}$. After some simple calculations, we derive
that the acceleration ratio of the strategy is at most $\frac{b^{rn+1}}{b-1}$, which in turn 
is minimized for $b=\frac{rn+1}{rn}$, and  which proves the claimed bound on the best exponential strategy.  

The non-monotone strategy is precisely the pseudo-exponential strategy presented in Theorem~\ref{thm:fault.contract}.
This strategy is strictly better than the best exponential strategy, since the function $f(x)=(1+\frac{1}{x})^x$ 
is increasing; however, the gap between the two strategies is small. In particular, for $n \rightarrow \infty$, both strategies 
converge to the same acceleration ratio.
\end{proof}
The strategies described above establish connections beyond those that result from the 
use of cyclic strategies. More precisely, we have shown that non-cyclic ray-searching 
algorithms have counterparts in the domain of contract-scheduling; furthermore, the non-cyclic strategies
improve upon the best cyclic ones. We have thus addressed an open question from~\cite{steins},
who asked whether there exist connections between the two problems that transcend cyclic strategies.


\section{Randomized scheduling of contract algorithms}
\label{sec:randomized}

In this section we study the power of randomization for scheduling (deterministic) contract 
algorithms. Our approach is motivated by the randomized strategy of~\cite{ray:2randomized}
for searching on $m$ rays. We emphasize, however, that our analysis differs in several key points, and most 
notably on the definition of appropriate random events.  


We will analyze the following randomized  strategy : We choose a random permutation $\pi:\{0, \ldots n-1 \} 
\rightarrow \{0, \ldots n-1\}$  of the $n$ problems, as well as a random $\varepsilon$ uniformly distributed in $[0,1)$. 
In every iteration $i \geq 0$, the algorithm executes a contract for problem  
$\pi(i) \mod n$ with corresponding length  $b^{1+\varepsilon}$, with $b>1$.

\begin{theorem}
The acceleration ratio of the randomized strategy is 
$\beta_r(n,b)= n \frac {b^{n+1} \ln b}{(b^n-1)(b-1)}$.
\label{thm:full.randomization.upper}
\end{theorem}
\begin{proof}
Let $t$ denote the interruption time. Observe that $t$ can be expressed as $t=\frac{b^k-1}{b-1} b^\delta$, for some unique $k \in \mathbb{N}$ and $\delta$ such that $1 \leq b^\delta <\frac{b^{k+1}-1}{b^k-1}$. For convenience,
we will call the contract execution of length $b^{i+\varepsilon}$ the {\em $i$-th contract} of the strategy, and $i$ the contract {\em index} 
(with $i \geq 0$). Note that the start and finish times of the $i$-th contract are $\frac{b^{i}-1}{b-1} b^\varepsilon$ and 
$\frac{b^{i+1}-1}{b-1} b^\varepsilon$, respectively.

First, we need to identify the index of the contract during the execution of which the interruption time $t$ occurs; denote this index by $l$. Note that it cannot be that 
$l \geq k+1$, since $\frac{b^{k+1}-1}{b-1}b^\varepsilon \geq 
\frac{b^{k+1}-1}{b-1} >t$. Similarly, it cannot be that $l \leq k-2$ because
$\frac{b^{k-1}-1}{b-1}b^\varepsilon \leq \frac{ {b^k}-b}{b-1}<\frac{b^k-1}{b-1}\leq t$.
We conclude that either $l=k$, or $l=k-1$. In particular, the random event $(l=k-1)$ occurs 
only when $\frac{b^{k}-1}{b-1} b^\varepsilon \geq t= \frac{b^{k}-1}{b-1} b^\delta$,
which implies that $\varepsilon \geq \delta$.

Next, we need to evaluate the expected value of the random variable $D$ that corresponds 
to the length of the longest contract for the problem that is requested at time  
$t$, and which has completed at time $t$. This will allow us to bound the acceleration ratio $\alpha$
of the randomized strategy, as 
\begin{equation}
\sup_t \frac{t} {\mathbb{E}[D]},  \ \textrm{with} \ t= \frac{b^k-1}{b-1} b^\delta \leq \frac{b^{k+1}-1}{b-1}.
\label{eq:randomization:acceleration}
\end{equation}

We consider two cases, depending on whether $\delta \geq1$.

\noindent
{\em Case 1: $\delta \geq 1$}. \ In this case, $\varepsilon < \delta$, which implies, from the above discussion that $k=l$. 
Therefore, the strategy will return one of the contracts with indices $k-1,k-2, \ldots, k-n$, namely the contract that 
corresponds to the requested problem. Due to the random permutation of problems performed by the strategy, each of these
indices is equally probable to correspond to the requested problem. We thus obtain
$
\mathbb{E}[D] = \mathbb{E}[D \mid (k=l)] = \frac{1}{n} \sum_{i=1}^n \mathbb{E}[b^{k-i+\varepsilon}] 
= \frac{1}{n} \sum_{i=1}^n b^{k-i}\frac{b-1}{\ln b} =
\frac{1}n{} \frac{b^k(b^n-1)}{b^n \ln b},
$
where we used  the fact that $\varepsilon$ is uniformly distributed in $[0,1)$.
Combining with~(\ref{eq:randomization:acceleration}) we obtain 
$\beta_r(n,b) \leq \frac{b^{k+1}-1}{b^-1} \frac{1}{\mathrm E[D]} \leq n \frac {b^{n+1} \ln b}{(b^n-1)(b-1)}$.

\noindent
{\em Case 2: $0 \leq \delta < 1$}, \ in other words, $b^\delta < b$. Note that in this case, the events 
$(l=k)$ and $(\varepsilon < \delta)$ are equivalent; similarly for the events $(l=k-1)$ and $(\varepsilon \geq \delta)$.
The following technical lemma establishes $\mathbb E[D]$ in this case.
\begin{lemma} 
$\mathbb{E}[D] = \frac{1}{n} \frac{b^{k-1}(b^n-1)b^\delta}{b^n} \ln b$.
\label{lemma:randomization.second.expectation}
\end{lemma}
\begin{proof}
Denote by $F$, and $\overline{F}$ the events $(l=k)$ and $(l=k-1)$, respectively. 
We have 
\begin{equation}
\mathbb{E}[D] = \mathbb{E}[D \mid F] \ \textrm{Pr($F$)}+
\mathbb{E}[D \mid \overline{F}] \ \textrm{Pr($\overline{F}$)}
\label{eq:randomization.upper.1}
\end{equation}
 Moreover,
 \begin{eqnarray}
\mathbb{E}[D \mid \overline{F}] &=& \frac{1}{n} \sum_{i=1}^n \mathbb{E}[b^{k-i+\varepsilon} \mid \overline{F}]  \nonumber \\
& =&\frac{1}{n} \frac{b^\delta-1} {\ln b} \frac{1} {\textrm{ Pr($\overline{F}$) }} \sum_{i=1}^n b^{k-i} \nonumber \\ 
&=& \frac{1}{n}\frac{b^\delta-1}{\ln b} \frac{b^{k-n}}{\textrm{ Pr($\overline{F}$) }} \frac{b^n-1}{b-1}. 
\label{eq:randomization:expectation1}
\end{eqnarray}
Similarly, we have that 
\begin{eqnarray}
\mathbb{E}[D \mid F] &=& \frac{1}{n} \sum_{i=1}^n \mathbb{E}[b^{k-1-i+\varepsilon} \mid F] \nonumber \\
&=&\frac{1}{n} \frac{b-b^\delta} {\ln b} \frac{1} {\textrm{ Pr(F) }}  \sum_{i=1}^n b^{k-1-i} \nonumber \\ 
&=& \frac{1}{n} \frac{b^\delta-1}{\ln b} \frac{b^{k-1-n}}{\textrm{ Pr(F) }} \frac{b^n-1}{b-1}. 
\label{eq:randomization:expectation2}
\end{eqnarray}
Here, we use the facts that 
\[
\mathbb{E}[b^\varepsilon \mid F] =\int_{b^\delta}^b x \cdot \frac{1} { \textrm{ Pr(F) } \ln b}, \ \textrm{and}
\]
\[
\mathbb{E}[b^\varepsilon \mid \overline{F}] =\int_{1}^{b^\delta} x \cdot \frac{1}{ \textrm{ Pr($\overline{F}$) } } \ln b.
\]
Combining~(\ref{eq:randomization.upper.1}),~(\ref{eq:randomization:expectation1}) and~(\ref{eq:randomization:expectation2})
we obtain, after some calculations, that 
\begin{equation}
\mathbb{E}[D] = \frac{1}{n} \frac{b^{k-1}(b^n-1)b^\delta}{b^n} \ln b,
\label{eq:randomization.second.expectation}
\end{equation}
which completes the proof.
\end{proof}

Combining Lemma~\ref{lemma:randomization.second.expectation} and~(\ref{eq:randomization:acceleration})
we obtain again that 
$\beta_r(n,b) \leq \frac{(b^k-1)b^\delta}{b^-1} \frac{1}{\mathrm E[D]} \leq n \frac {b^{n+1} \ln b}{(b^n-1)(b-1)}$.
\end{proof}

%

\subsection{Evaluation of the randomized strategy}
\label{subsec:randomized.evaluation}
In order to evaluate the best randomized exponential strategy, we must find the $b$
that minimizes the function $\beta_r(n,b)$.
It is easy to see, using standard calculus, that $\beta_r(n,b)$ has a unique minimum, for given $n$.
However,  unlike the deterministic case, there is no closed form for $\beta_r^*(n)=\min_{b>1} \beta_r(n,b)$. 
Thus, we must resort to numerical methods.

Figure~\ref{fig:randomization} illustrates the performance of the randomized strategy $\beta_r^*(n)$ versus
the deterministic optimal strategy, denoted by $\beta^*(n)$. We observe that $\beta^*_r(n) \leq 0.6\beta^*(n)$,
for $n=1, \ldots 80$.
In fact, we can show analytically that for $n\rightarrow \infty$, $\beta_r^*(n)$ converges to a value that does not exceed 
$\frac{\mathrm e}{\mathrm e-1}(n+1)$  (recall that $\beta^*(n)$ converges to $\mathrm e (n+1)$). More precisely, choosing 
$b=\frac{n+1}{n}$ we obtain $\beta_r^*(n) \leq (n+1)\frac{(1+1/n)^n \ln(1+1/n)}{((1+1/n)^n-1)(1+1/n)}$,
which converges to $(n+1)\frac{\mathrm e}{\mathrm e -1}$, a value extremely close to the computational results.

\begin{figure}[htb!]
   \centerline{\includegraphics[height=7cm, width=9 cm]{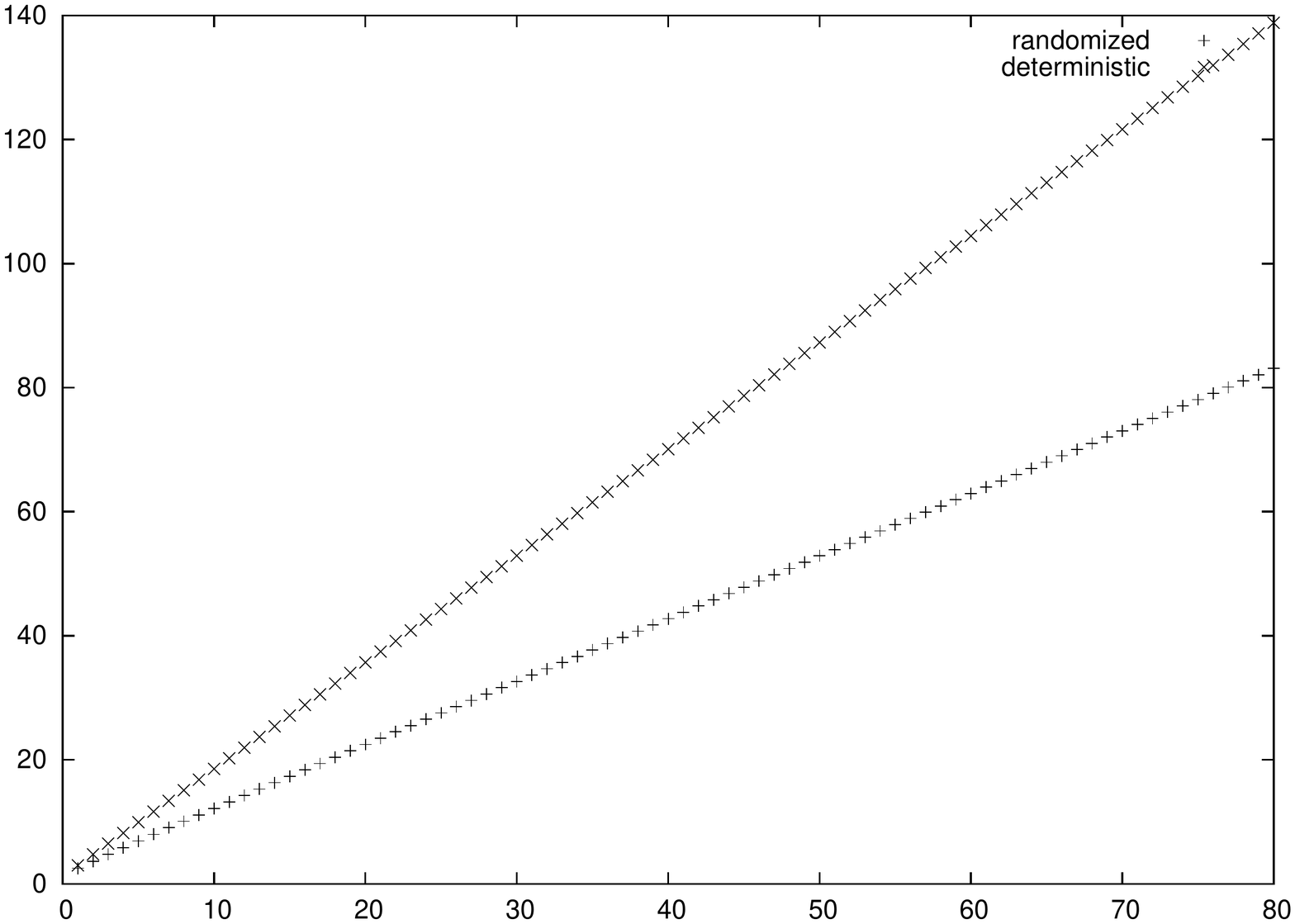}}
  \caption{Plots of the randomized ($\beta_r^*(n))$ and the deterministic ($\beta^*(n)$) 
 acceleration ratios, as functions of $n$.}
  \label{fig:randomization}
\end{figure}

\section{Trade-offs between performance and executions of searches/algorithms}
\label{sec:preemptions}

Most previous work on ray searching assumes that the searcher
can switch directions at no cost. In practice, turning is a costly operation in robotics, 
and thus should not be ignored. In a similar vein, we usually
assume that there is no setup cost upon execution of a contract algorithm, however
some initialization cost may be incurred in practice. One could address this requirement by 
incorporating the turn/setup cost in the performance
evaluation (see~\cite{demaine:turn} for ray searching with turn cost). 
In this section we follow a different approach by studying the trade-off between performance
and the number of searches and/or executions of algorithms. 
 
We will make a distinction between two possible settings. In the first setting, we use the standard
definitions of search and scheduling as given in Section~\ref{sec:introduction}. Specifically, we 
address the question: Given a target at distance $t$ (resp. an interruption $t$) what is the minimum
number of turns (resp. executions of contracts) so as to guarantee a certain competitive ratio 
(resp. acceleration ratio)? We call this the {\em standard} model.

The second setting is motivated by applications in which searching previously explored
territory comes at no cost. One such example is the expanding search paradigm~\cite{thomas:expanding}.
Another example is parallel linear searching on arrays modeled as ray searching~\cite{hyperbolic}, in which 
the searcher can ``jump'' to the last-explored position.

While the latter setting does not have a true counterpart in the realm of contract scheduling, it still gives rise 
to a scheduling problem. Suppose we have $n$ problems, each with its own statement of an 
{\em interruptible} algorithm (as opposed to a contract algorithm). In addition, we allow the use of 
{\em preemptions}, in that we can preempt, and later resume the execution of an algorithm. In this context,
we face the scheduling problem of interleaving the executions of interruptible algorithms. Note that we 
can still use the acceleration ratio, given by~(\ref{eq:acceleration.ratio})) as the
performance measure, with the notable difference that here $\ell_{i,t}$ denotes the {\em total} (aggregate) time of algorithm
executions for problem $i$, by time $t$. 
We call the above model the {\em preemptive} model. 


\subsection{Trade offs in the preemptive model}
\label{subsec:preemptive}

We consider first the problem of scheduling interleaved executions of interruptible algorithms. Clearly, 
the optimal acceleration ratio is $n$: simply assign each time unit uniformly across all problems, in 
a round-robin fashion. However, this optimal strategy results in a linear number of preemptions, as function of time. 
We thus consider the following {\em geometric} round-robin strategy, which is a combination of uniform and 
exponential strategies. The strategy works in phases; namely, in phase $i$ ($i \geq 0$), it
executes algorithms for problems $0 \ldots n-1$  with each algorithm allotted a time span 
equal to $b^i$, for fixed $b>1$ (we will call each algorithm execution for problem $i$ a {\em job} for problem $i$).
\begin{lemma}
The geometric strategy has (worst-case) acceleration ratio $n(b+1)$, asymptotic acceleration ratio $nb$, and for any
$t$, the number of preemptions incurred up to $t$ is at most $n \log_b \left(\frac{t(b-1)}{n}+1 \right) +n$.
\label{lemma:preemptions.upper}
\end{lemma}
\begin{proof}
The worst-case acceleration ratio of the geometric strategy is attained at interruptions right before the end of a phase, say phase $i$,
in other words, for interruption time $t=n\frac{b^{i+1}-1}{b-1}$. At this time, every problem has been completed to an aggregate 
job length equal to $\ell=\sum_{j=0}^{i-1} b^j=\frac{b^i-1}{b-1}$. It is very easy to verify that $\frac{t}{l} \leq n(b+1)$, and 
that $\frac{t}{l} \rightarrow nb$, as $t \rightarrow \infty$ (i.e., for $i \rightarrow \infty$).

We now focus on bounding the number of preemptions. Suppose that 
the interruption $t$ occurs in the $i$-th phase, then we can express $t$ as 
$t=n \sum_{j=0}^{i-1}b^j+xnb^i$, where $x\in[0,1)$, therefore we obtain that $\frac{t}{n} \geq \frac{b^i-1}{b-1}$,
and hence $i \leq \log_b \left(\frac{t(b-1)}{n}+1 \right)$.
On the other hand, the number of interruptions by time $t$ is $I_t \leq in+\lceil xn \rceil \leq n(i+1)$. The result follows.
\end{proof}

We will now show that the geometric strategy attains essentially the optimal trade-offs. 
\begin{theorem}
For any strategy with (worst-case) acceleration ratio $n(1+b)-\epsilon$ for any $b>1$, and constant $\epsilon>0$, 
there exists $t$ such that the number of  preemptions up to time $t$ is at least $n \log_b \left(\frac{t(b-1)}{n}+1 \right) -n$. 
Moreover, any strategy with asymptotic acceleration ratio $nb(1-\epsilon)$, for any constant $\epsilon>0$, incurs 
$n \log_b \left(\frac{t(b-1)}{n}+1 \right) -o\left(n \log_b \left(\frac{t(b-1)}{n}+1 \right)\right)$ preemptions by time $t>t_0$,
for some $t_0$.
\label{thm:preemptions.lower}
\end{theorem}
\begin{proof}
For the first part of the theorem, suppose, that a strategy $S$ has (worst-case) acceleration ratio $\beta=n(b+1)-\epsilon$,
and incurs fewer than $n \log_b \left(\frac{t(b-1)}{n}+1 \right) -n$ preemptions for any $t$. 
We will first show that there exists a strategy $S'$ with the following properties: i) at its first phase, $S$ executes $n$
jobs, all of the same unit length, for each of of the $n$ problems; ii) the number of preemptions of $S'$ at time $t$
does not exceed the number of preemptions of $S$ by more than $n$; and iii) $S'$ has no worse acceleration ratio than $S$.
To see this, we will use the canonical assumption that interruptions occur only after at least a job per problem has been
executed (see~\cite{steins}). Let $l_1,l_2, \ldots ,l_n$ denote the aggregate lengths of jobs in this first phase, in 
non-decreasing order; here $l_i>1$, for all $i$ (since we may assume, from normalization, that the smallest job length is equal to 1).
Consider then a strategy $S'$ which first schedules $n$ unit jobs, one per problem, followed by $n$ more jobs (again one per problem)
of lengths $l_1-1, \ldots l_n-1$. From that point onwards, $S'$ is precisely $S$. In other words, $S'$ is derived by substituting the
initial phase of $S$ by two sub-phases, as defined above. It is easy to see that $S'$ has no worse acceleration ratio than $S$.
Moreover, since $S'$ introduces at most $n$ new job executions, in comparison to $S$.  Therefore, $S'$ is such that at time $t$ at most 
$n \log_b \left(\frac{t(b-1)}{n}+1 \right)$ preemptions are incurred. 

Let $t$ be arbitrarily close to, but smaller than $nl(b+1)$. 
Then from the assumption, $S'$ must incur fewer than $n \log_b\left( \frac{n(b^2-1)}{n}+1\right)$ preemptions by time $t$. This would imply
that there is a problem for which $S'$ does not schedule a job within the interval $[n, (b+1)n]$, from which it follows that the acceleration
ratio of $S'$ is at least $n(b+1)$, since at time $t$ there is a problem that has been executed to aggregate length equal to 1, which 
is a contradiction.

For the second part of the theorem, fix a strategy $S$ of asymptotic acceleration ratio $\beta=nb(1-\epsilon)$. 
Consider a partition of the timeline in phases, 
such that the $i$-th phase ($i \geq 0$) spans the interval $[n\sum_{j=0}^{i-1}b^j, n\sum_{j=0}^i b^j)$, and thus has length $nb^i$.
We will show that there exists $i_0>0$ such that for all $i \geq i_0$, $S$ must incur at least $n$ preemptions in its $i$-th phase. 
Since the geometric strategy with base $b$ incurs exactly $n$ preemptions in this interval, for all $i$, 
this will imply that we can partition the timeline $t \geq i_0$ in intervals with the property that in each interval, 
$S$ incurs at least as many preemptions as the geometric strategy, which suffices to prove the result. 

Suppose, by way of contradiction, that $S$ incurred at most $n-1$ preemptions within $T=[n\sum_{j=0}^{i-1} b^j, n\sum_{j=0}^{i} b^j]$.
Therefore, there exists at least one problem $p$ with no execution in $T$. Consider an interruption at time 
$t=n\sum_{j=0}^{i} b^j -\delta$, for arbitrarily small $\delta>0$.  Thus, the aggregate job length for $p$ by time 
$t$ in $S$ is $\ell_{p,t}\leq n \sum_{j=0}^{i-1} b^j=n\frac{b^{i}-1}{b-1}$. Since $S$ has asymptotic 
acceleration ratio $\beta$, there must exist $i_0$ and $ \epsilon'$ with $0<\epsilon'<\epsilon$ such that for all $i \geq i_0$,
$
n\frac{b^{i+1}-1}{b-1} -\delta \leq nb(1-\epsilon') \frac{b^{i}-1}{b-1},
$
which it turn implies that 
$\epsilon' \frac{b^{i}-b}{b-1} \leq \delta$ for all $i>i_0$. This is a contradiction, since $\epsilon'$ 
depends only on $i_0$, and $\delta$ can be arbitrarily small.  
\end{proof}
Next, we consider ray-searching and the trade-offs between the competitive ratio and the number of turns. 
Recall that in the model we study, the searcher incurs cost only upon visiting newly explored territory. 
In particular, we define the geometric search strategy as a round-robin search of the rays; more precisely,
in the $i$-th phase of the strategy each ray is searched up to distance $b^i$ from the origin. 
\begin{theorem}
The geometric search strategy has (worst-case) competitive ratio $(b+1)m$, asymptotic competitive ratio $bm$, 
and is such that if the searcher had incurred cost $d$, the overall number of turns is at most
$m \log_b \left(\frac{d(b-1)}{m}+1 \right)+m$.
Moreover, for any search strategy with (worst-case) acceleration ratio $m(1+b)-\epsilon$ for any $b>1$, and constant $\epsilon>0$, 
there exists a target placement such that the searcher incurs cost $d$, and the number of  
number of turns is at least $m \log_b \left(\frac{d(b-1)}{m}+1 \right) -m$. 
Last, any strategy with asymptotic competitive ratio $mb(1-\epsilon)$, for any constant $\epsilon>0$, makes
$m \log_b \left(\frac{d(b-1)}{m}+1 \right) -o\left(m \log_b \left(\frac{d(b-1)}{m}+1 \right)\right)$ 
turns for search cost $d>d_0$, for some $d_0$.
\label{thm:rays.preemption}
\end{theorem}  
\begin{proof}
The proof follows by arguments very similar to the proofs of Lemma~\ref{lemma:preemptions.upper} and
Theorem~\ref{thm:preemptions.lower}. 
Concerning the performance of geometric search, 
let $0, \ldots ,m-1$ denote the rays visited (in round-robin
order) during each phase. We note that the worst-case placement of the target is attained at points right after the 
turn point of the searcher in the end of phase $i$, 
and in particular after the searcher has incurred cost $d=m\frac{b^{i+1}-1}{b-1}$, whereas the distance of the target from the origin is
equal to $\sum_{j=0}^{i-1}b^j=\frac{b^i-1}{b-1}$. The bounds on the competitive ratio and the number of turns follow
then by the arguments in Lemma~\ref{lemma:preemptions.upper}. 

The trade-offs between the competitive ratio and the number of turns follow from ideas very similar to Theorem~\ref{thm:preemptions.lower}.
More precisely, suppose that a strategy $S$ has (worst-case) competitive ratio $\alpha=m(b+1)-\epsilon$,
and incurs fewer than $m \log_b \left(\frac{d(b-1)}{m}+1 \right) -m$ turns, where $d$ is the search cost. 
We can show that there exists a strategy $S'$ of competitive ratio at most $\alpha$, with at most
$m \log_b \left(\frac{d(b-1)}{m}+1 \right)$ turns for some $d$, and is such that in its initial phase, each 
ray is searched up to distance 1 from the origin. We then use strategy $S'$ to derive a contradiction 
(this construction is as in the proof of Theorem~\ref{thm:preemptions.lower}).

Last, we can show the claimed trade-off between competitive ratio and the number of turns.
Fix a strategy $S$ of asymptotic competitive ratio $\alpha=mb(1-\epsilon)$. 
Consider a partition of the timeline in phases, such that the $i$-th phase ($i \geq 0$) spans the interval 
$[m\sum_{j=0}^{i-1}b^j, m\sum_{j=0}^i b^j)$, and thus has length $mb^i$. Since the searcher has unit speed, we obtain the same partition
concerning the cost incurred by the searcher. 
We will show that there exists $i_0>0$ such that for all $i \geq i_0$, $S$ must make at least $m$ turns in its $i$-th phase. 
Since the geometric strategy with base $b$ makes exactly $m$ turns in this interval, for all $i$, 
this will imply that we can partition the cost incurred by the searcher $d \geq d_0$ in intervals with the property that in each interval, 
$S$ incurs at least as many turns as the geometric strategy, which suffices to prove the result. 

Suppose, by way of contradiction, that $S$ made at most $m-1$ turns within $D=[m\sum_{j=0}^{i-1} b^j, m\sum_{j=0}^{i} b^j]$.
Therefore, there exists at least one ray $r$ which was not searched in $D$. This implies that at time
$t=m\sum_{j=0}^{i} b^j -\delta$, for arbitrarily small $\delta>0$, there is a ray that has not been searched to depth more than 
$m \sum_{j=0}^{i-1} b^j=m\frac{b^{i}-1}{b-1}$. The remainder of the proof follows precisely as the proof of 
Theorem~\ref{thm:preemptions.lower}, by considering a target on ray $r$ placed at distance $m\frac{b^{i}-1}{b-1}+\epsilon$ from the origin. 
\end{proof}

\subsection{Trade offs in the standard model}
\label{subsec:standard}
The ideas of Section~\ref{subsec:preemptive} can also be applied in the standard
model. In this setting, however, exponential strategies are a more suitable candidate. 
\begin{theorem}
\noindent
For contract scheduling, the exponential strategy with base $b$ has acceleration ratio $\frac{b^{n+1}}{b-1}$,
and schedules at most $\log_b (t(b-1)+1)+1$ contracts by $t$. Moreover, any strategy with acceleration
ratio at most $\frac{b^{n+1}}{b-1}-\epsilon$ for $b>1$, and any $\epsilon>0$ must schedule at least 
$\log_b (t(b-1)+1)-o(\log_b (t(b-1)+1))$ contracts by $t$, for all $t \geq t_0$.
\label{thm:scheduling.standard}
\end{theorem}
\begin{proof}
It is known that any exponential strategy with base $b$ has acceleration ratio $\frac{b^{n+1}}{b-1}$. 
If an interruption $t$ occurs during the $i$-th execution of a contract, then $t \geq \sum_{j=0}^{i-1}b^j=\frac{b^i-1}{b-1}$.
Thus, $i \leq \log_b(t(b-1)+1)$, and since the number of contracts by time $t$ is at most $i+1$, we obtain the desired upper bound. 

For the lower bound, we will show a result even stronger than claimed in the statement of the theorem. More precisely, we 
will show that any schedule $S$ with acceleration ratio $\beta=\frac{b^{n+1}}{b-1}$ 
must schedule at least $n+1$ contracts in the timespan $T=[\sum_{j=0}^{i-n-1} b^j, \sum_{j=0}^{i} b^j]$, for all $i$
(we will thus allow even $\epsilon=0$).
Since the exponential strategy with base $b$ schedules exactly $n+1$ contracts in this interval, this will imply that 
we can partition the timeline in intervals with the property that in each interval, $S$ schedules at least as 
many contracts as the exponential strategy, which suffices to prove the result. 

Suppose, by way of contradiction, that $S$ scheduled at most $n$ contracts in the timespan $T=[\sum_{j=0}^{i-n-1} b^j, \sum_{j=0}^{i} b^j]$. 
Therefore, at most one contract for each problem has been executed in this interval. 
Consider an interruption at time 
$t=\sum_{j=0}^{i} b^j -\delta$, for arbitrarily small $\delta>0$. From the assumption, there is at least one problem $p$
for which $S$ did not {\em complete} any contract in the time span $[\sum_{j=0}^{i-n-1} b^j, \sum_{j=0}^{i} b^j-\delta]$.
Thus, the largest contract for $p$ that has completed by time $t$ in $S$
can have length at most $l=\sum_{j=0}^{i-n-1} b^j=\frac{b^{i-n}-1}{b-1}$. Since $S$ has acceleration ratio $\beta$, it must be that
$t \leq \beta l$, which gives
\[
\frac{b^{i+1}-1}{b-1} -\delta \leq \frac{b^{n+1}}{b-1} \frac{b^{i-n}-1}{b-1},
\]
which it turn implies that $\delta \geq \frac{b^{n+1}}{(b-1)^2}-\frac{1}{b-1}$. 
This is a contradiction, since $\delta$ can be chosen to be arbitrarily small, and in particular, smaller than  
$\frac{b^{n+1}}{(b-1)^2}-\frac{1}{b-1}$. 
\end{proof}
For ray searching in the standard model, we can obtain similar results. 
Recall that in this model, the searcher incurs cost at all times it moves, 
regardless of whether it explores new territory. 
We can prove the following theorem along the lines of the proof of Theorem~\ref{thm:scheduling.standard}.

\begin{theorem}
For ray searching in the standard model, the exponential strategy with base $b$ has competitive ratio $1+2\frac{b^{m}-1}{b-1}$, and for any distance 
$d$ traversed by the searcher it makes at most  $\log_b (d(b-1)+1)+1$ turns.  
Moreover, any strategy with competitive ratio at most  $1+2\frac{b^{m}-1}{b-1}$, for any $b>1$ 
incurs at least $\log_b (d(b-1)+1)+1-o(\log_b (d(b-1)+1)$ turns. 
\label{thm:searching.standard}
\end{theorem}
\begin{proof}
It is known that the exponential strategy with base $b$ has competitive ratio $1+2\frac{b^{m}-1}{b-1}$. For given distance $d$
traversed by the searcher, the number of turns is computed using precisely the same argument as the number of contract executions of the 
exponential strategy in the proof of Theorem~\ref{thm:scheduling.standard}. Likewise, for the lower bound, 
we will show that any strategy $S$ with competitive ratio $\alpha=1+2\frac{b^{m}}{b-1}$ must make at least $m$ 
turns in the timespan $[\sum_{j=0}^{i-m} b^j, \sum_{j=0}^{i} b^j]$
(recall that since the searcher has unit speed, time coincides with the distance traversed by the searcher).
Since the exponential strategy with base $b$
searches exactly $m$ rays in this interval, this will imply that we can partition the timeline (and thus the distances traversed by the searcher)
in intervals with the property that in each interval, $S$ searches at least as many rays as the exponential 
strategy, which suffices to prove the result. 

Suppose, by way of contradiction, that $S$ made fewer than $m$ turns in the timespan $T=[\sum_{j=0}^{i-m} b^j, \sum_{j=0}^{i} b^j]$.
Therefore, there exists a ray that has not been searched in $T$, say ray $r$. Let $\rho$ denote the depth at which $r$ has been 
searched up to time $\sum_{j=0}^{i-m} b^j$, and consider a target placement in $r$ at distance $\rho+\delta$, for 
arbitrarily small $\delta$. For this target placement, the competitive ratio is minimized when $\rho$ is maximized; moreover, since
$r$ was not searched in $T$, we obtain that $\rho$ is at most $\frac{1}{2}\sum_{j=0}^{i-m} b^j=\frac{b^{i-m+1}-1}{2(b-1)}$.
Here, the factor $1/2$ is due to the search traversing each ray in both directions (away and towards the origin). 
Note also that the target is discovered, at the earliest, at time $\sum_{j=0}^i b^i+\rho+\delta=\frac{b^{i+1}-1}{b-1}+\rho+\delta$.
Since the strategy has competitive ratio $\alpha=1+2\frac{b^{m}}{b-1}$, it must be that 
\[
\frac{b^{i+1}-1}{b-1}+\rho+\delta \leq (1+2\frac{b^{m}}{b-1}) (\rho+\delta),
\]
from which, after some simplifications, we obtain that 
$2\frac{b^m}{b-1} \delta \geq \frac{b^m}{(b-1)^2}-\frac{1}{b-1}$, which is a contradiction, since 
$\delta$ can be arbitrarily small.
\end{proof}

\section{Conclusion}
\label{sec:conclusion}
In this paper we demonstrated that many variants of searching for 
a target on concurrent rays and scheduling contract algorithms on a single processor
are amenable to a common approach. 
There are some intriguing questions that remain open. 
Can we obtain a $\Theta(m/p)$-competitive algorithm for searching with probabilistic detection? 
We believe that cyclic strategies are not better than $\Theta(m/p^2)$-competitive.
What are the optimal (non-monotone) algorithms for searching/scheduling 
with redundancy? Note that the precise competitive ratio of these problems is open even when $m=2$. As a broader
research direction, it would be very interesting to address searching and scheduling in heterogeneous environments. 
For example, one may consider the setting in which each ray is characterized by its own probability of successful target detection.

\bibliographystyle{named}
\bibliography{targets}

\end{document}